%% file: root.tex
\documentclass[letter, 10pt, conference]{ieeeconf}      

\IEEEoverridecommandlockouts                              
\usepackage{amsthm, amssymb, amsmath}
\usepackage{mathrsfs}
\newtheorem{prob}{Problem}
\newtheorem{prop}{Proposition}
\newtheorem{thm}{Theorem}

\newtheorem{assum}{Assumption}

\usepackage{pgfplots} 
\usepackage{subfig}
\usetikzlibrary{arrows.meta}
\tikzset{>=latex}
\usepackage{colortbl}
\input{macros.tex}

\usepackage{booktabs}
\usepackage{ltl}

\title{Near-Optimal Reactive Synthesis Incorporating Runtime
    Information\thanks{This material is based upon work
        supported by the Office of Naval Research
        (N00014-18-1-2829), National Aeronautics and Space
Administration (80NSSC19K0209), and U.S. Army Research
Laboratory (ACC-APG-RTP W911NF).}}
\author{Suda Bharadwaj\textsuperscript{1}, Abraham P. Vinod\textsuperscript{1}, Rayna Dimitrova\textsuperscript{2}, Ufuk Topcu\textsuperscript{1}\vspace{0.3cm} \\
  \normalsize\textsuperscript{1}The University of Texas at Austin \\
  \textsuperscript{2}The University of Sheffield, UK \\}

\begin{document}
\date{}
\maketitle

\begin{abstract}
We consider the problem of optimal reactive synthesis --- compute a strategy that satisfies a mission specification in a dynamic environment, and optimizes a performance metric. We incorporate task-critical information, that is only available at runtime, into the strategy synthesis in order to improve performance. Existing approaches to utilising such time-varying information require online re-synthesis, which is not computationally feasible in real-time applications. In this paper, we pre-synthesize a set of strategies corresponding to \emph{candidate instantiations} (pre-specified representative information scenarios). We then propose a novel switching mechanism to dynamically switch between the strategies at runtime while guaranteeing all safety and liveness goals are met. We also characterize bounds on the performance suboptimality. We demonstrate our approach on two examples --- robotic motion planning where the likelihood of the position of the robot's goal is updated in real-time, and an air traffic management problem for urban air mobility.
\end{abstract}


\section{Introduction}

As autonomous systems become more widely used in society, we require provable guarantees of performance and safety in complex missions~\cite{belta2007symbolic,finucane2010ltlmop}. In many applications, it is not enough for an autonomous agent to satisfy its mission objective, but it is often required that it also optimizes some performance metric. Due to limits on communication, sensing, or computational power, the autonomous agent may have access to information that may be available only at the time of execution. Traditional approaches either ignore this information or can only make use of it at the cost of heavy computation or high memory requirements~\cite{Ehlerscost,jangcontinuous}. We propose a correct-by-construction switching strategy that utilizes such information at runtime for improved performance while guaranteeing the satisfaction of high-level mission specifications, and also alleviates the shortcomings of the existing methods to enable real-time deployment. 

For example, consider a motion-planning problem for a service robot as shown in Figure~\ref{fig:gazeboworld}. A high-level mission for the robot is to meet the human infinitely often, while ensuring that it always has sufficient battery power (rechargeable by returning to a charging station). Given the probability of the human's location based on past observations (runtime information), the proposed approach finds the human in a shorter period of time (compared to strategies that ignore this probability information), while satisfying the safety specification. 

For another, more complex example, consider the coordination of landing a collection of autonomous air vehicles in urban air mobility (UAM) operations \cite{goyal2018urban,gipson2017nasa}. 
We seek to optimize performance (reduce  maximum delay in aircraft landing) while ensuring safe takeoff and landing operations \cite{thipphavong2018urban}. The on-demand nature of UAM means knowledge of air traffic demands is not available at design time. This necessitates a method that can use traffic information gained at runtime to adjust behavior for improved performance and safety. Previous approaches implement runtime safety enforcement \cite{bhnfm, AlshiekhShield,Konighofer2017,8815233}, but cannot handle more general specifications. 

\begin{figure}
    \input{figs/gazeboworld.tex}
    \caption{Path planning environment for a turtlebot (in blue) to infinitely often service a human (in red). The robot can recharge at a charging station (in green).}
    \label{fig:gazeboworld}
    \vspace{-0.5cm}
\end{figure}
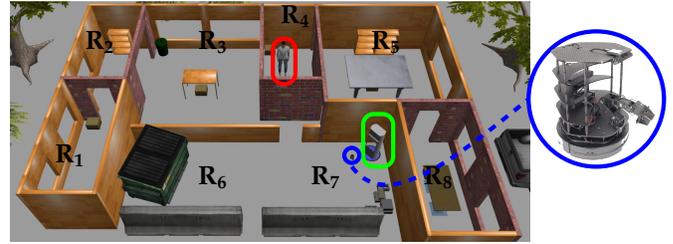

The two scenarios described previously illustrate a \emph{reactive} planning problem. The autonomous system has to react to an uncontrolled environment, and guarantee correctness with respect to a given mission specification for all possible behaviours of the environment for all time points in the future. The standard approach to solve such a planning problem is to use \emph{reactive synthesis} \cite{piterman2006,bloem2012}. In particular, there is a large body of work focused on synthesis for a fragment of linear temporal logic (LTL) called GR(1)~\cite{Ehlerslugs,bh18,Alonso18,Moarref18}. The solution time is polynomial in the state space of the game structure, and exponential in the number of atomic propositions. Therefore, this approach typically relies on offline planning, that prevents easy incorporation of runtime information. In problems with a continuous state space, a discrete abstraction is used that preserves correctness. Such controllers however, can be significantly suboptimal with respect to the performance objective~\cite{jangcontinuous}. 

We consider the \emph{runtime information} as a (possibly continuous) parameter associated with the environment. The work in~\cite{jangcontinuous} allows for near-optimal behaviour on continuous executions, however the authors focus on a specialized cost metric. Additionally, their method relies on online re-synthesis, which is not feasible for real-time deployment. This work was later extended in~\cite{Ehlerscost} to account for delay costs arising from a potentially adversarial environment. However, it relies on the discretization of the continuous parameter space, which fails to scale with the number of atomic propositions in the synthesis problem. 

Our approach incorporates parametrized runtime information by switching between pre-computed strategies. First, for a given set of \emph{candidate instantiations} of the parameter we synthesize offline optimal strategies that satisfy all task specifications. Next, we obtain bounds on the suboptimality incurred by the use of these policies at \emph{all} other parameter values. This computation does not require discretization of the parameter space. At runtime, we dynamically switch strategies based on the these suboptimality bounds, thereby incorporating runtime information into the offline synthesis of correct-by construction policies. 
To this end, we derive a switching function that guarantees the resulting execution is provably correct, and near-optimal.

The main contributions of this paper are: 1) a novel switching protocol between pre-synthesized correct strategies that improves performance, 2) correctness of the switching protocol with respect to the mission specification, and 3) characterization of the suboptimality bounds on performance. We demonstrate the proposed approach on a motion planning problem for a service robot, and a traffic scheduling problem for UAM. 


\section{Preliminaries}\label{sec_prel}

\begin{figure*}[t!]
\centering
    \subfloat[$\overline{p}_1 = \lbrack1,0,0\rbrack $ \label{fig:gazebopolicies1}]{\includegraphics[width=0.33\linewidth]{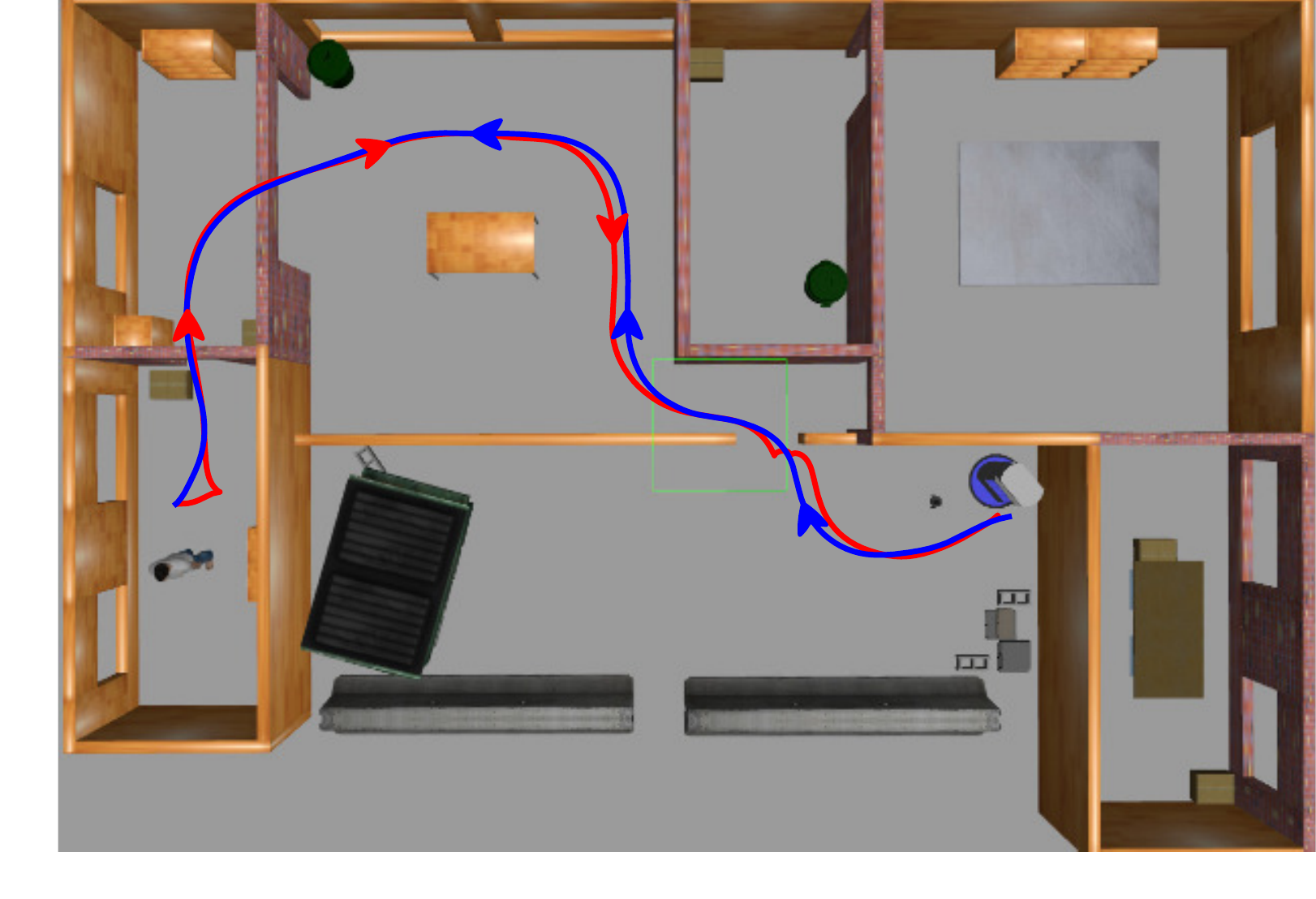}}
    \subfloat[$\overline{p}_2= \lbrack0,1,0\rbrack $]{\includegraphics[width=0.33\linewidth]{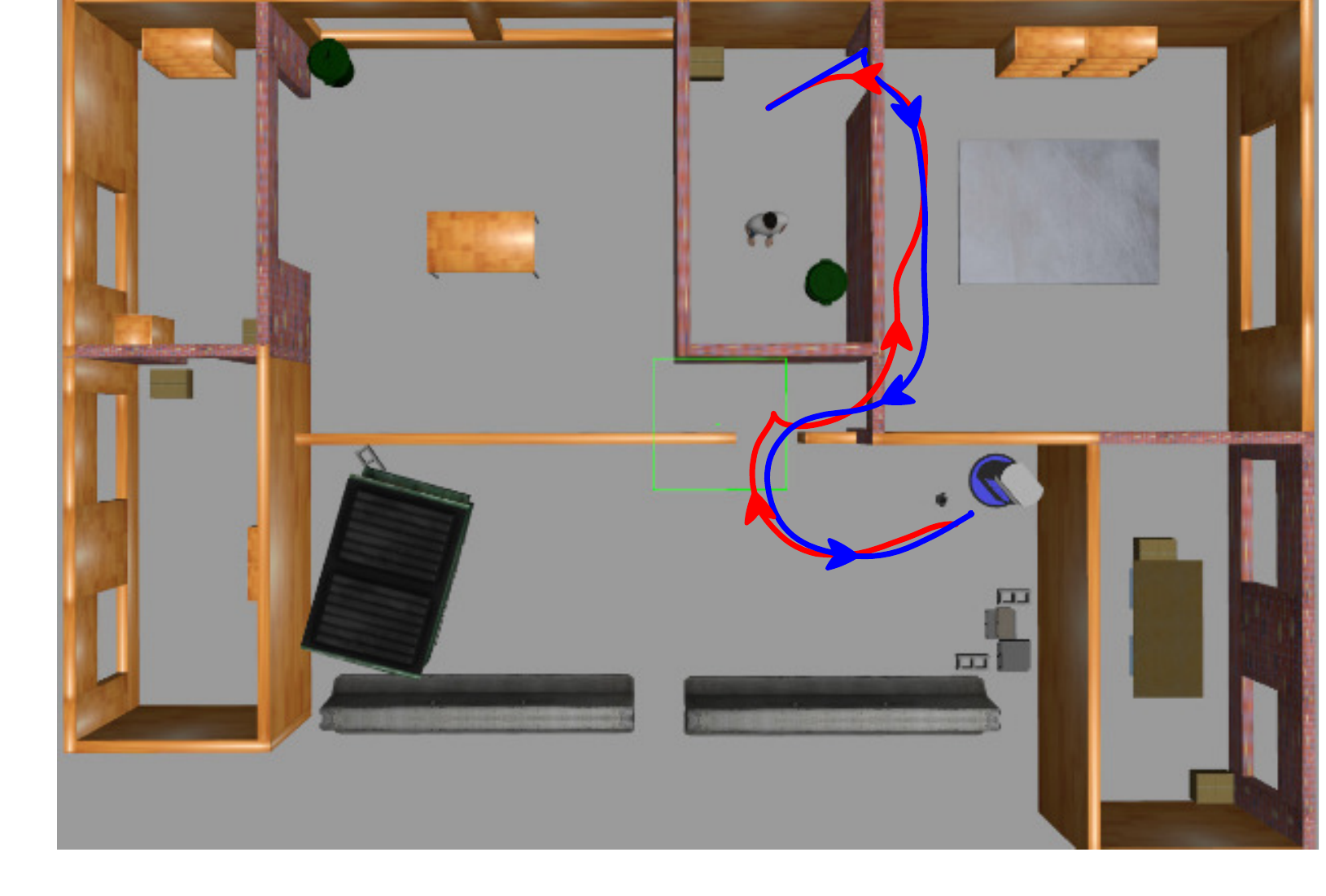}}
    \subfloat[$\overline{p}_3 = \lbrack0,0,1\rbrack $]{\includegraphics[width=0.33\linewidth]{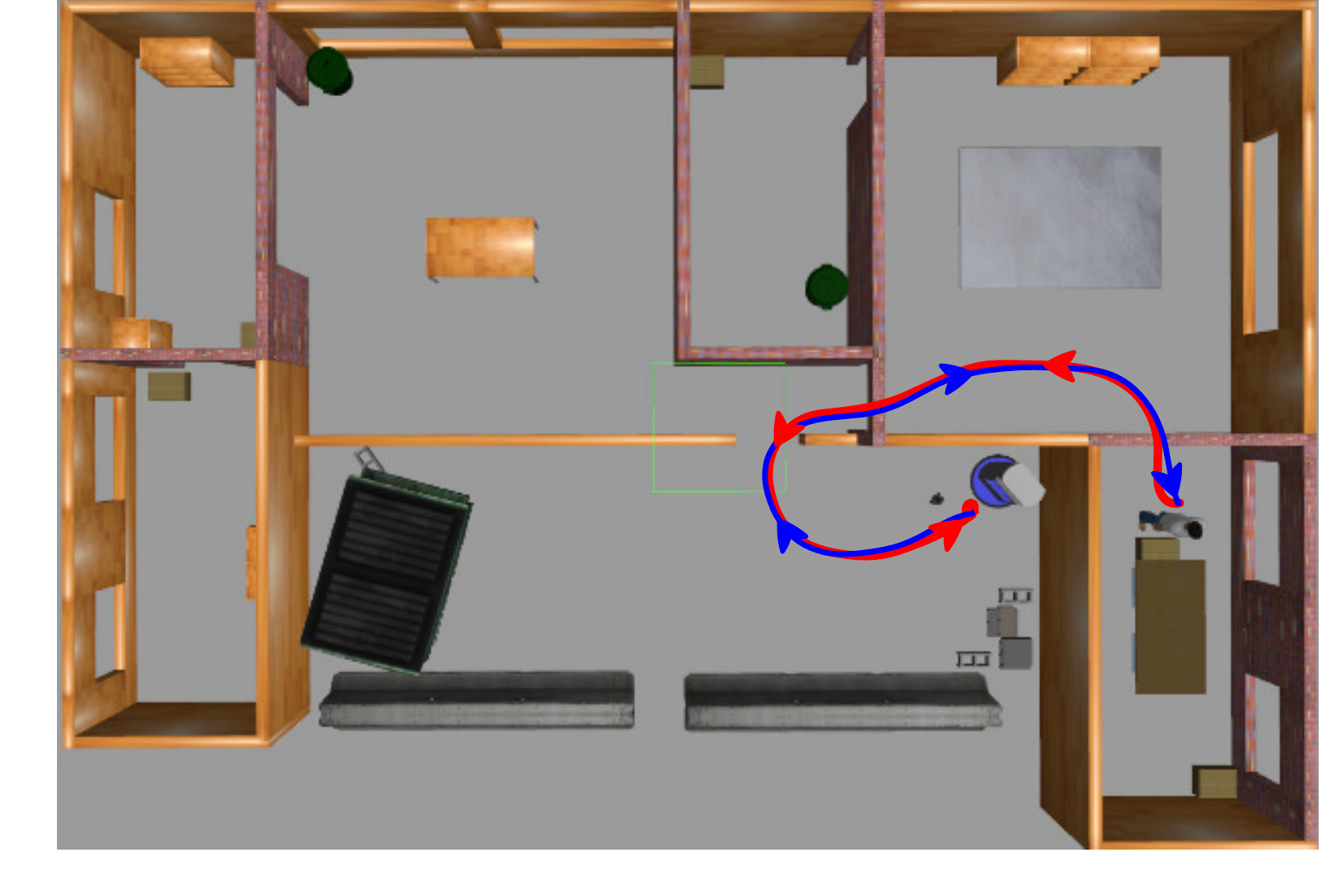}}
    \caption{Continuous trajectories resulting from executing policies corresponding to the solution of~\eqref{prob:opt} for three different instantiations of runtime information vector $\overline{p} \colon \overline{p}_1 = [1,0,0]$, $\overline{p}_2 = [0,1,0]$, and $\overline{p}_3 = [0,0,1]$ .}\label{fig:trivialtrajs}
\end{figure*}

\input{preliminaries.tex}


\section{Problem formulation}\label{sec:prob}

\input{prob_st.tex}

\section{Synthesis of correct-by-construction strategies with near-optimality guarantees}

We address Problem~\ref{prob_st:main} by constructing a switching function that guarantees that the resulting plays satisfy the task specification $\spec$. We also provide suboptimality bounds for the performance when using the proposed method. We utilize existing synthesis techniques~\cite{Ehlerscost} to synthesize for each $i \in \{1,\ldots,N\}$ a strategy that is \emph{correct}, i.e., satisfies $\spec$, and optimal for the specific $\overline{p}_i$.

\subsection{Suboptimality bounds for unknown runtime information}
\label{sec:generalization}
\input{partition.tex}

\subsection{Switching function construction}\label{sec:switching}
\input{switching.tex}

\subsection{Discussion}
\input{Switching_Discussion.tex}

\section{Experiments}
All experiments we report on were performed on an Intel i5-5300U 2.30 GHz CPU with 8 GB of RAM.  We used the tool \texttt{Slugs}~\cite{Ehlerslugs} for the strategy synthesis. 

\subsection{Robot motion planning}

We consider the example discussed in Section~\ref{sec:prob} (Figure~\ref{fig:gazeboworld}).
Formally, the specification is
\[
    \varphi = \LTLglobally (h \in R_1 \cup R_4 \cup R_8) \rightarrow \Big( \LTLglobally \LTLfinally (r = h) \wedge \LTLglobally (\mathrm{Energy} > 0) \Big),
\]
where $h$ and $r$ are variables modelling the human and robot positions respectively, and $\mathrm{Energy}$ is the robot's energy level.
The cost function $C(\cdot)$ is given in \eqref{eq:cost_fun}.
The runtime information $\overline{p}$ is the probability distribution over the human's possible locations - $R_1, R_4,R_8$. 
In our experiments, we used a Bayesian update to compute $\overline{p}$ using the current (and past) observations of the human's position. 

\begin{table}
    \centering
    \begin{tabular}{|c|c|c|c|c|}\toprule
        Query point  & \multirow{2}{*}{Strategy} & \multicolumn{3}{c|}{Cost}  \\\cmidrule{3-5}
        $\overline{p}$  & & L. bound & U. bound &  $C^\ast(\overline{p})$\\  \midrule
        $[0.1, 0.8, 0.1]$ & $\rho_{\out_2}$ & 7.19 & 10.4 & 9.5 \\ 
        $[0.0, 0.1, 0.9]$ & $\rho_{\out_3}$  & 6.39 & 9.6 & 7.9 \\ 
        \rowcolor{red!10} $[0.6, 0.3, 0.1]$ & -- & -- & -- & 11.1\\ \bottomrule
    \end{tabular}
    \caption{Lower and upper bounds (Theorem~\ref{thm:bounds}) and the optimal value $C^\ast(\overline{p})$ for some runtime information vectors $\overline{p}$.}
    \label{tab:perf}
\end{table}{}
\begin{figure}
    \centering
    \input{figs/presentationfigure}
    \caption{State space partitioning of the runtime information vector $\overline{p}$ for $\epsilon = 3.2$. For runtime information vector belong to the darker shaded regions, we obtain $\epsilon$-optimality by reusing a specific strategy $\rho_{\out_{(\cdot)}}$, and avoid computationally expensive re-synthesis.}
    \label{fig:partition}
\end{figure}
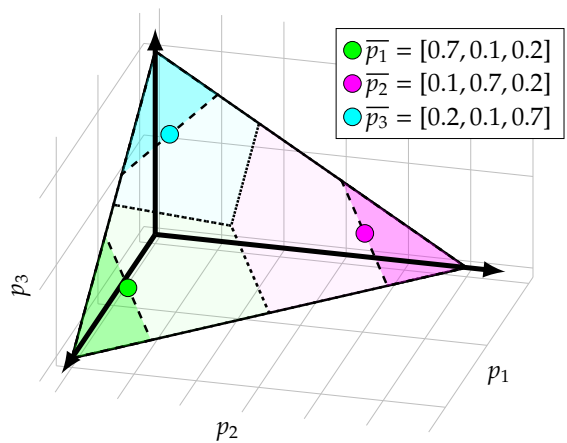



We used three candidate instantiations of $\overline{p}$ (Figure~\ref{fig:partition}). The robot only has enough charge to visit one of the three rooms and return to the charging station. The optimal robot strategy for each $\overline p_i$ corresponds to an \emph{ordering} of which room to visit. Intuitively, the robot will visit rooms in decreasing order of likelihood of a human being there, by executing the continuous trajectories shown in Figure~\ref{fig:trivialtrajs}.

Figure~\ref{fig:partition} shows the partition of the space of $\mathcal{P}$ generated from the corresponding polytopes $\mathcal{S}_i$ for $\epsilon = 3.2$. The choice of $\epsilon$ is dictated by Proposition~\ref{prop:nonempty}. The three candidate instantiations of $\overline{p}$ are represented by colored dots. The darker coloured regions are the polytopes $\mathcal{S}_i$, and they correspond to the regions of $\mathcal{P}$ in which the corresponding strategy is $\epsilon$-optimal. Note that $\mathcal{P}\not\subseteq\cup_{i=1}^N\mathcal{S}_i$, and there are parameters in $\mathcal{P}$ were none of the three strategies are $\epsilon$-optimal. The light shaded areas around each $\overline{p}_i$ corresponds to the portion of the parameter space in which the correspondingly coloured strategy dominates the others, but is not $\epsilon$-optimal.

Table~\ref{tab:perf} shows the proposed optimality bounds obtained from our approach (Theorem~\ref{thm:bounds}). On comparing with the lowest delay possible (computed via re-synthesis), we see that the computed bounds holds in the first two rows.   
The last row has $\overline{p} = \lbrack 0.6,0.3,0.1 \rbrack$ lying outside $\cup_{i=1}^N\mathcal{S}_i$ (outside of the dark shaded region in Figure~\ref{fig:partition}), has no informative bounds. Here, $\rho_{\out_1}$ is the dominating strategy among $\rho_{\out_{(\cdot)}}$, with $C(\rho_{\out_1},\overline{p})=13.6$.



A video of the simulation of the robot meeting the human (infinitely often) as the human moves in real-time can be found at \url{https://youtu.be/pn6afwf5INc}.

\subsection{Urban air mobility traffic management}

We now consider an automated air traffic management system for urban air mobility (UAM) operations. The controller is required to optimize the throughput of a multi-pad UAM port, along with bounding the delays experienced by vehicles and passengers. We synthesize a controller for a UAM hub, which consists of a grouping of multiple UAM vertiports. The hub has restrictions on the number of aircraft it is allowed to simultaneously land across all vertiports. Hence, a controller, if necessary, must make incoming air vehicles wait until it is able to safely allow them to land. In this example, we consider three vertiports --- A ($red$), B ($yellow$), and C ($blue$) where an aircraft can request to land. Formally, the task specification is
\vspace{-0.12cm}
\begin{multline*}
    \varphi = \LTLglobally (\text{Current Requests} < R) 
\rightarrow \\ \qquad\;\LTLglobally (\text{No. Landing Aircraft} < M) \; \wedge \\ \LTLglobally (\text{Land Request} \rightarrow \LTLfinally \text{Land Allowed}),
\end{multline*}
where $R$ is maximum number of simultaneous requests and $M$ is the maximum number of aircraft allowed to land simultaneously. We model incoming aircraft as landing requests for vertiports drawn from a time-varying probability distribution. We model the performance metric as the \emph{maximum delay}. The cost function is
\begin{align}
C(\rho_A,\overline{p}) = \max_i(\text{delay}(V_i,\rho_\out))\cdot q_i\label{eq:cost_UAM}
\end{align}
where $\overline{p} = \lbrack q_1 , q_2,q_3,q_4 \rbrack$, $q_i$ is the probability of a request to land at vertiport hub $V_i$ for $i=\{1,2,3\}$, and $q_4$ is the probability of no landing request, and $\text{delay}(V_i,\rho_\out)$ is the processing delay at $V_i$ under strategy $\rho_\out$. We pre-compute strategies for three representative instantiations of the runtime information with each strategy taking 213 s to compute.
Initially, we choose the true distribution to be one of the instantiations. At $t=150$ minutes, the underlying probability distribution of landing requests changes such that the uninformed strategy performs poorly and the new probability value is not part of any of the representative instantiations. At $t=400$ minutes, the probability distribution switches back to the initial probability distribution. The uninformed strategy is a fixed, runtime information-independent strategy that satisfies $\spec$.



\begin{figure}
    \centering
    \input{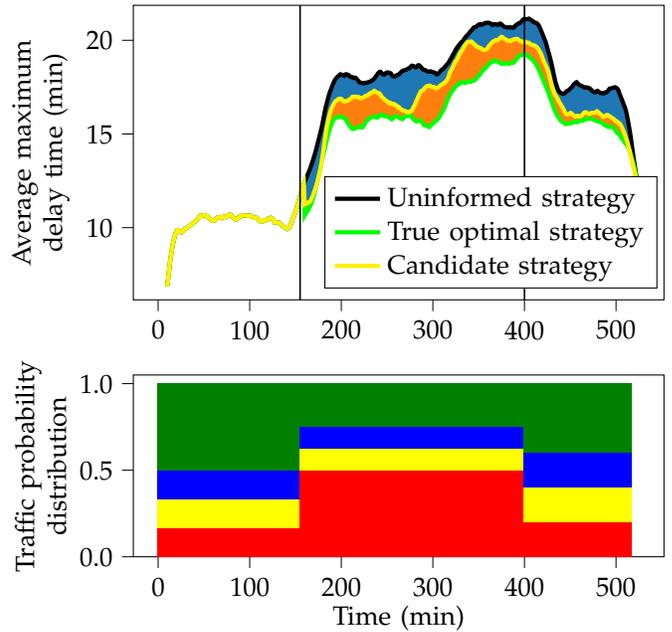}
    \caption{Top: Average maximum delay times for $N=100$ runs with landing requests drawn from a time varying probability distribution shown in bottom figure. Black vertical lines corresponds to a switch in strategy. Below: The probability of an aircraft requesting to land at hubs A ($red$), B ($yellow$), C ($blue$), or not land ($green$).}
    \label{fig:delay}
\end{figure}

Figure~\ref{fig:delay} compares the proposed switching strategy against an optimal strategy that relies on re-synthesis (requires heavy computation) and an uninformed strategy (does not incorporate runtime information). 
Initially, during low traffic times, we see no deviation in delay times as expected. The proposed switching strategy provides a significant reduction of delay over time compared to an uninformed strategy. However, it is suboptimal to the strategy obtained via re-synthesis, which relies on heavy computation that rules out real-time execution. 
Since \eqref{eq:cost_UAM} does not satisfy Assumption~\ref{assum:C_lin}, we do not have suboptimality bounds on the performance (Theorem~\ref{thm:bounds}). Empirically, the proposed approach shows an improvement in performance.

\section{Conclusion and future work}
We present a method to integrate information about environment behaviour gained at runtime into reactive synthesis. Our technique provides significant performance gains over standard reactive synthesis without sacrificing any correctness or facing state space explosion. In future work we intend to investigate the use of counterexamples to generate more candidate instantiations of runtime information parameters in order to guarantee $\epsilon$-optimality over the entire parameter set $\mathcal{P}$. 

\clearpage
\bibliographystyle{IEEEtran}
\bibliography{references}

\end{document}

%% file: macros.tex
\newcommand{\comment}[1]{}

\usepackage{amsfonts}
\usepackage{amssymb}
\usepackage{amsmath}
\usepackage{caption}
\usepackage{cite}
\usepackage{color}
\usepackage{epstopdf}
\usepackage{float}
\usepackage{graphicx}
\usepackage{ltl}
\usepackage{pxfonts}
\usepackage{tikz}
\usepackage{xspace}
\usepackage{multirow}
\usepackage{hyperref}

\newcommand{\game}{\mathcal{G}}
\newcommand{\gstates}{G}
\newcommand{\ginit}{g_0}

\newcommand{\spec}{\varphi}

\newcommand{\plays}{\mathit{Plays}}
\newcommand{\prefs}{\mathit{Prefs}}

\newcommand{\inp}{\textit{E}}
\newcommand{\out}{\textit{A}}

\newcommand{\alphabet}{\Sigma}
\newcommand{\ialphabet}{\Sigma_\inp}
\newcommand{\oalphabet}{\Sigma_\out}

\newcommand{\symb}{\sigma}

\newtheorem{definition}{Definition}

%% file: figs/gazeboworld.tex
\begin{tikzpicture}
    \node[anchor=south west,inner sep=0] at (0,0) {\includegraphics[width=0.8\linewidth]{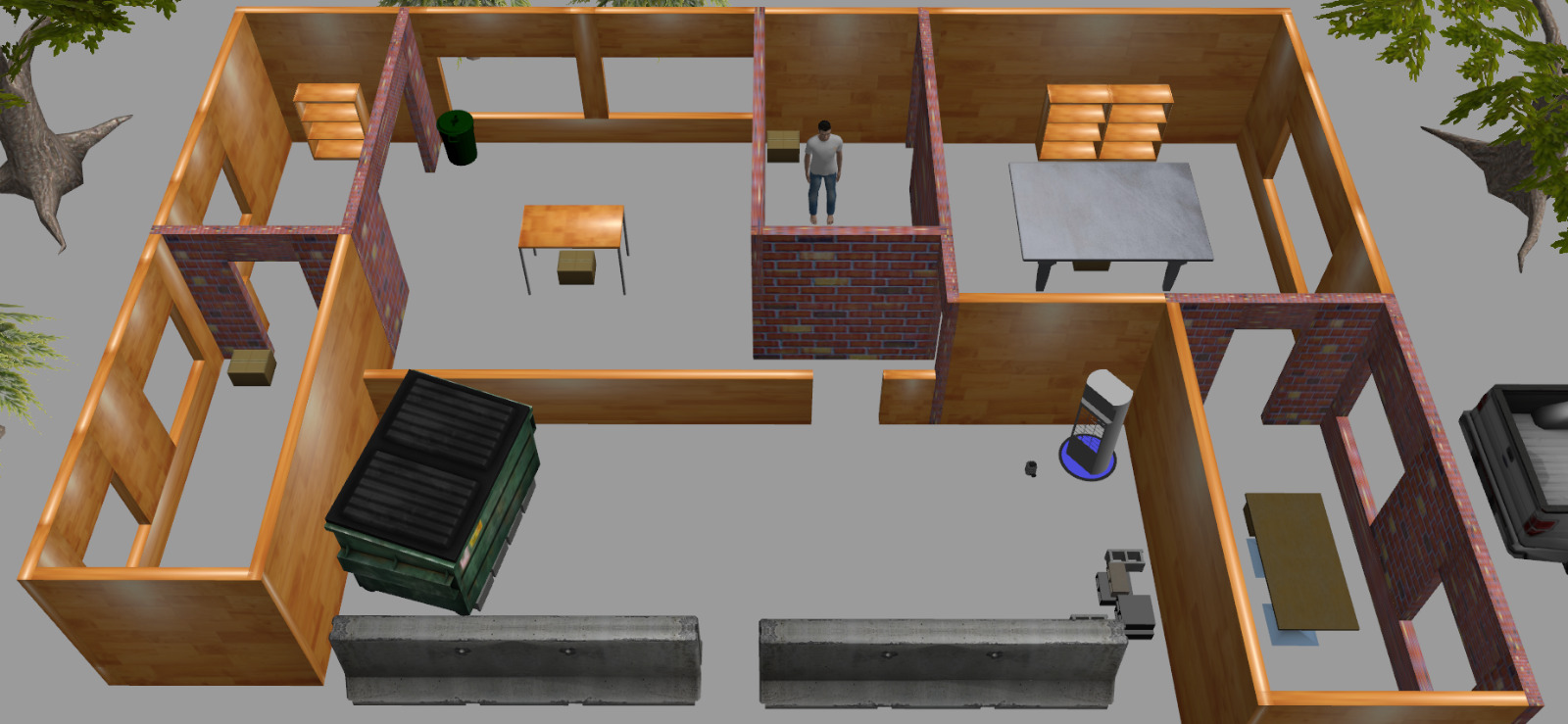}};
    \draw[red,ultra thick,rounded corners] (3.5,2.1) rectangle (3.8,2.7);
    \draw[green,ultra thick,rounded corners] (4.7,1.0) rectangle (5.1,1.7);
    \draw[blue,ultra thick] (7.8,1.9) circle (9mm);
    \node[anchor=south west,inner sep=0] at (0.6,1.0){$\mathbf{R_1}$};
    \node[anchor=south west,inner sep=0] at (1.0,2.5){$\mathbf{R_2}$};
    \node[anchor=south west,inner sep=0] at (2.5,2.5){$\mathbf{R_3}$};
    \node[anchor=south west,inner sep=0] at (3.6,2.85){$\mathbf{R_4}$};
    \node[anchor=south west,inner sep=0] at (4.8,2.5){$\mathbf{R_5}$};
    \node[anchor=south west,inner sep=0] at (2.5,0.7){$\mathbf{R_6}$};
    \node[anchor=south west,inner sep=0] at (4.0,0.7){$\mathbf{R_7}$};
    \node[anchor=south west,inner sep=0] at (5.5,0.7){$\mathbf{R_8}$};
    \node[anchor=south west,inner sep=0] at (7,1)
    {\includegraphics[width=0.2\linewidth]{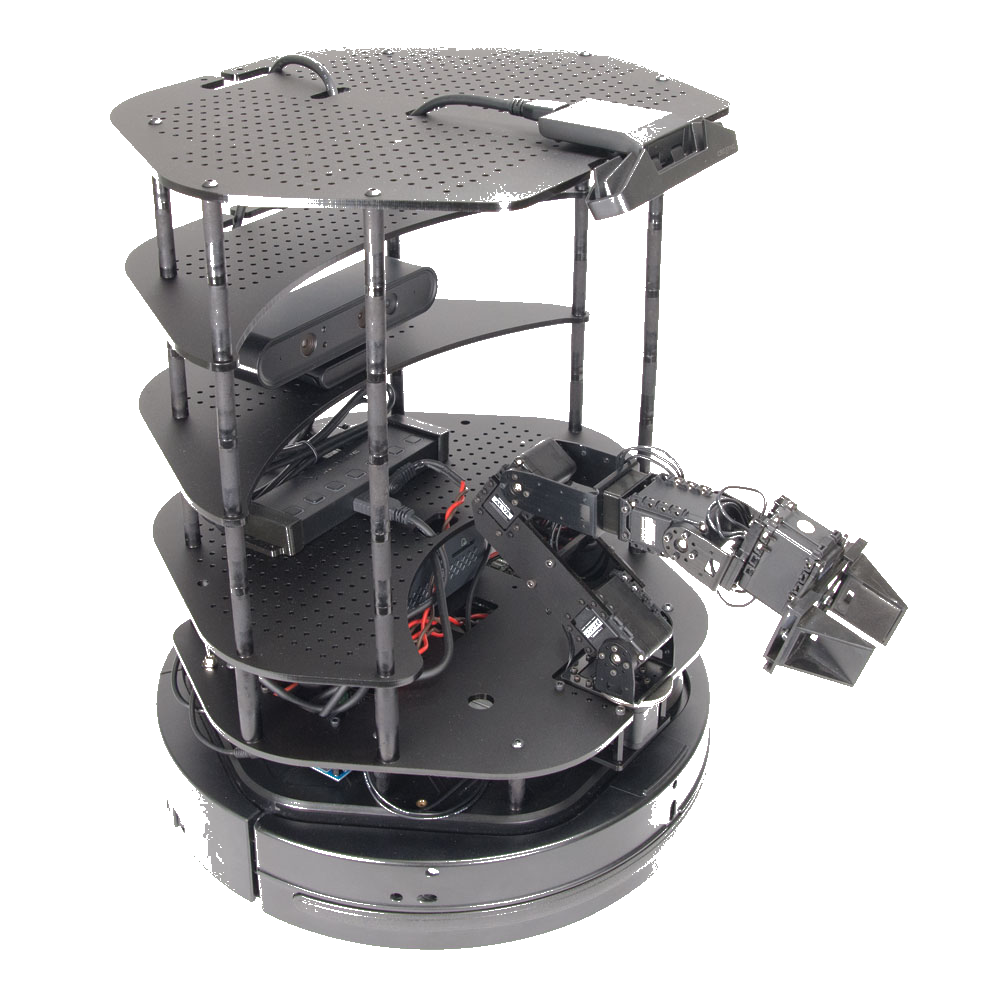}};
    \draw[blue,ultra thick] (4.55,1.15) circle (1mm);
         \draw[blue, ultra thick,dashed] (4.55,1.0) ..  controls (5.0,0.4) and (5.8,1) .. (6.9,1.9);
\end{tikzpicture}

%% file: preliminaries.tex
\subsubsection{\textbf{Basic notation}}
We consider reactive systems with a finite set $\inp$ of Boolean \emph{inputs}, controlled by the \emph{E}nvironment, and a finite set $\out$ of Boolean \emph{ outputs}, controlled by the \emph{A}gent. Together, they define the system's input alphabet $\ialphabet=2^\inp$ and the output alphabet $\oalphabet=2^\out$. We define $\alphabet=\ialphabet \times \oalphabet$. 

\subsubsection{\textbf{Game structures}}
We model the interaction between the agent and its environment as a two-player game. Formally, the game is played on a \emph{game structure} which is a tuple 
$\game = (\gstates, \ginit, \alphabet, \delta)$,
where:
\begin{itemize}
\item $\gstates$ is a finite set of states and $\ginit \in \gstates$ the initial state;
\item $\alphabet = \ialphabet \times \oalphabet$ is the alphabet of actions available to the environment and the agent respectively;
\item $\delta: \gstates \times \alphabet \rightarrow \gstates$
is a complete transition function, that maps each state, input (environment action) and output (agent action) to a successor state.
\end{itemize}

\subsubsection{\textbf{Winning conditions}} 
The \emph{winning condition} for the agent in a game $\game$ is given as a set of plays $\varphi \subseteq \plays(\game)$ that specifies the set of plays that result in the agent winning the game. We consider games in which the agent has a \emph{Generalized Reactivity 1} (GR(1)) winning condition, which are common in a variety of practical applications. In the following, we make use of the linear temporal logic (LTL) operators \emph{always} $\LTLglobally$ and \emph{eventually} $\LTLfinally$. For full details on LTL syntax and semantics, we refer the reader to~\cite{MCBook}.

A GR(1) winning condition is defined by sets of states $S_\inp, S_\out \subseteq G$, $E_i \subseteq G$ for $i=1,\ldots,m$ and $F_j \subseteq G$ for $j=1,\ldots,n$, and consists of all plays $\overline \pi$ such that if $\overline{\pi} \in \LTLglobally S_\inp \cap \LTLglobally\,\LTLfinally\, E_{i}$ for all $i=1,\ldots,m$, then $\overline{\pi} \in \LTLglobally S_\out \cap \LTLglobally\,\LTLfinally\, F_{j}$ for all $j=1,\ldots,n$. Intuitively, for a play $\overline \pi$ to be winning, whenever the environment satisfies the assumptions $\LTLglobally\, S_\inp,\LTLglobally\,\LTLfinally\, E_{1},\ldots,\LTLglobally\,\LTLfinally\, E_{m}$, then the agent must satisfy all the guarantees $\LTLglobally\, S_\out,\LTLglobally\,\LTLfinally\, F_{j},\ldots,\LTLglobally\,\LTLfinally\, F_{n}$. By abuse of logical operators, we abbreviate GR(1)  conditions as
$$\varphi =  \left(\LTLglobally\, S_\inp \wedge \bigwedge_{i=1}^{m}  \LTLglobally\,\LTLfinally\, E_{i}\right) \rightarrow
\left(\LTLglobally\, S_\out \wedge \bigwedge_{i=1}^{n} \LTLglobally\,\LTLfinally\,F_{i}\right).$$

\subsubsection{\textbf{Strategies}}
A \emph{strategy for the agent} is a function $\rho_\out:
\prefs(\game) \times \ialphabet \rightarrow
\oalphabet$ which maps a prefix (the history of the play so far) and an action of the environment to an action of the agent. 
A \emph{strategy for the environment} is a function $\rho_\inp: \prefs(\game)\rightarrow \ialphabet$ that maps the prefix of the play so far to an action of the environment. We denote the sets of all strategies for the agent and for the environment by $\mathcal{M}_\out $ and $\mathcal{M}_\inp$ respectively.

Every pair of strategies $\rho_\out \in \mathcal{M}_\out$ for the agent and $\rho_\inp \in \mathcal{M}_\inp$ for the environment define a play, denoted by $\Pi(\rho_\out,\rho_\in)$. More precisely,  
$\Pi(\rho_\out,\rho_\in) = \overline{\pi} = (g_0,\symb_{\inp,0},\symb_{\out,0}, g_1) 
(g_1,\symb_{\inp,1}, \symb_{\out,1}, g_2) \ldots \in \plays(\game)$
where
for every $i \geq 0$, $\symb_{\inp,i} = \rho_\inp(\overline\pi[0,i])$ and $\symb_{\out,i} = \rho_\out(\overline\pi[0,i],\symb_{\inp,i})$.
Similarly, we define the set of plays starting at a state $g$ that are consistent with $\rho_\out$, denoted $\plays(\game,\rho_\out,g)$.

Given a game structure $\game$ and a winning condition $\varphi$ for the agent, we seek to synthesize a strategy $\rho_\out\in \mathcal{M}_\out$ for the agent such that for every strategy $\rho_\inp \in \mathcal{M}_\inp$ for the environment it holds that $\Pi(\rho_\inp,\rho_\out) \in \varphi$, i.e., all resulting plays satisfy $\varphi$.
In such cases we say that \emph{$\rho_\out$ satisfies $\spec$}, denoted $\rho_\out\models\spec$.

%% file: prob_st.tex
We represent \emph{runtime information} as $n$-dimensional real vectors, for a given $n \in \mathbb N$. We denote the set of all possible vector values for the runtime information by $\mathcal{P}\subseteq \mathbb{R}^n$. We score the performance of each play in the game using runtime information in $\mathcal{P}$ via a performance metric, $J: \plays(\game)\times\mathcal{P} \to \mathbb{R}$.


\begin{assum}\label{ass:max-cost}
For every $\overline p \in \mathcal{P}$ and every strategy $\rho_\out\in\mathcal{M}_\out$ such that $\rho_\out\models\spec$, there exists a strategy $\rho_\inp \in\mathcal{M}_\inp$ such that $  J(\Pi(\rho_\out,\rho_\inp), \overline p) \geq J(\Pi(\rho_\out,\rho_\inp'), \overline p)$ for every $\rho_\inp'\in\mathcal{M}_\inp$. 
\end{assum}
Assumption~\ref{ass:max-cost} ensures a well-defined cost function using the metric $J$. We can thus define
\begin{align}
C(\rho_\out,\overline{p}) = max_{\rho_\inp \in\mathcal{M}_\inp}J(\Pi(\rho_\out,\rho_\inp), \overline p)
\end{align}
as the cost function, with $C:\Sigma_\out \times \mathcal{P} \to \mathbb{R}$.
Given the runtime information $\overline p \in \mathcal P$, a strategy $\rho_\out \in \mathcal M_\out$ for the agent that satisfies $\spec$ is \emph{optimal for $\overline p$} if and only if it is a solution to the following optimization problem.
\begin{align}
    \begin{array}{rl}
        \underset{\rho_\out\in
        \mathcal{M}_\out}{\mathrm{minimize}}&\quad
        \ \\
        \mathrm{subject\ to}&\quad \rho_\out\models \varphi
    \end{array}\label{prob:opt}%
\end{align}
Let $C^\ast: \mathcal{P} \to \mathbb{R}$ denote the optimal value of \eqref{prob:opt}.

\begin{itshape}
\textbf{Example.} Consider Figure~\ref{fig:gazeboworld}, where the robot has to infinitely often meet the human. Assume that the human can only be in rooms $R_1, R_4, R_8$. Let $q_1,q_2,q_3\in[0,1]$ be the probabilities of the human being in room $R_1$, $R_4$ and $R_8$ respectively. The runtime information is $\overline{p}= {[q_1\ q_2\ q_3]}^\top \in\mathcal{P}\subseteq \mathbb{R}^3$, where the set $\mathcal{P}$ is the probability simplex. The cost function is, 
\begin{align}
C(\rho_\out, \overline{p}) &= \mathbb E\lbrack\text{time to find human\rbrack}= \sum\nolimits_{i=1}^N T_i(\rho_\out)\, q_i\label{eq:cost_fun}
\end{align}
where $T_i$ is the time taken to reach room $i$ under the robot strategy $\rho_\out$. Figure \ref{fig:trivialtrajs} shows the resulting continuous trajectories from executing policies when the information vector tells the robot the exact room occupied by the human. \label{ex:toy_ex2}
\end{itshape}



Given a set of representative strategies associated with instances of the runtime information, the task at runtime then becomes one of choosing a strategy depending on the current value of $\overline p$. As we have a finite set of strategies to choose from, the resulting behaviour of the agent will be \emph{approximately optimal}. Thus, we consider the problem of synthesizing an \emph{approximately optimal switching function} that also guarantees $\spec$.

\begin{definition}
Given a game structure $\game$ and a set of strategies $\{{\rho_\out}_i\in \mathcal{M}_\out\}_{i=1}^N$ for the agent, a \emph{switching function} is a function $\tau : \prefs(\game) \times \mathcal P^* \to \{1,\ldots,N\}$, which maps a play prefix and the sequence of values of $\overline p$ seen so far, to an index of a strategy in the given set.

The set of plays resulting from applying the switching function $\tau$ to $\{{\rho_\out}_i\in \mathcal{M}_\out\}_{i=1}^{N}$ is defined as the set of plays 
$\plays(\{{\rho_\out}_i\in\mathcal{M}_\out\}_{i=1}^{N},\tau)$ such that 
$\overline{\pi} = (g_0,\symb_{\inp,0},\symb_{\out,0}, g_1) 
(g_1,\symb_{\inp,1}, \symb_{\out,1}, g_2) \ldots \in \plays(\{{\rho_\out}_i\in \mathcal{M}_\out\}_{i=1}^{N},\tau)$
if and only if there exists $\overline \gamma \in \mathcal{P}^\omega$ such that
for every $i \geq 0$,  it holds that $\symb_{\out,i} ={\rho_\out}_{\tau(\overline \pi[0,i],\overline \gamma[0,i])}(\pi[0,i],\symb_{\inp,i})$.
\end{definition}

Informally, we want to be able to \emph{switch} between pre-computed strategies based on values of the runtime information. In order to not violate the specification, switching needs to take into account the prefix of the play. We formalize this task below. 

\begin{prob}\label{prob_st:main}
   We are given a game $(\game,\varphi)$, 
   a set $\{\overline{p}_i\in \mathcal{P}\}_{i=1}^{N}$ of representative values of the runtime information, and 
    strategies $\{{\rho_\out^\ast}_i\in \mathcal{M}_\out\}_{i=1}^{N}$ such that ${\rho_\out^\ast}_i$ solves \eqref{prob:opt} for $\overline{p}_i$. 

   Given $\epsilon > 0$, compute a collection $\{\mathcal S_i\}_{i=1}^N$ of subsets of $\mathbb{R}^n$ such that $\overline{p}_i \in \mathcal{S}_i$, and a  switching function $\tau : \prefs(\game) \times \mathcal P^+ \to \{1,\ldots,N\}$ that satisfies the following conditions.
   \begin{itemize}
       \item For all $i =1,\ldots,N$ and all $\overline p \in \mathcal S_i \cap \mathcal P$ it holds that
    \begin{align}
  C({\rho_\out}_i^\ast,
    \overline{p}) - \epsilon \leq
        C^\ast(\overline{p}) \leq
        C({\rho_\out}_i^\ast,
    \overline{p}).
    \label{eq:bounds}
   \end{align}  
   \item $\overline\pi \in \varphi$ for every play $\overline \pi \in \plays(\{{\rho_\out}_i\in\mathcal{M}_\out\}_{i=1}^{N},\tau)$.
   \end{itemize}

\end{prob}

%% file: partition.tex
Let $\{\overline{p}_i\in \mathcal{P}\}^{N}_{i=1}$ be the given candidate instantiations of the runtime information, and let 
$\{{\rho_\out^\ast}_i\in \mathcal{M}_\out\}_{i=1}^{N}$ be the corresponding optimal strategies. That is, for each $i$, ${\rho_\out}_i^\ast$ is an optimal solution of \eqref{prob:opt} for $\overline{p}_i$. 
Under Assumption~\ref{assum:C_lin} below, we can compute a collection of polytopes ${\{\mathcal{S}_i\}}_{i=1}^N$ such that ${\rho_\out^\ast}_i$ is $\epsilon$-optimal, whenever $\overline p \in \mathcal{S}_i$.
\begin{assum}\label{assum:C_lin}
  The cost function $C: \mathcal M_\out\times \mathcal{P}
  \to \mathbb{R}$ is such that for all $\rho_\out \in \mathcal M_\out$ and $\overline p = {[q_1\ q_2\ \ldots\ q_n]}^\top \in \mathcal P \subseteq \mathbb{R}^n$:
  \begin{align}
        C(\rho_\out, \overline{p}) = \sum_{i=1}^n
        C(\rho_\out, \overline{e}_i)\, q_i,\label{eq:V_lin_decompose}
    \end{align}
    where $\overline{e}_i$ is the vector that has $1$ at
     position $i$ and $0$ elsewhere.
\end{assum}

We address Problem~\ref{prob_st:main} by defining a collection of polytopes $\{\mathcal{S}_i\}_{i=1}^N$, where $\mathcal{S}_i= \{ \overline{p} \in \mathbb{R}^n \mid H_i \overline{p}\leq
\overline{b}_i\}$ and
{\small\begin{align}
    H_i &= \left[
        \begin{array}{ccc}
            C({\rho_\out}_i^\ast, \overline{e}_1) -
            C({\rho_\out}_1^\ast, \overline{e}_1) &
            \cdots &
            C({\rho_\out}_i^\ast, \overline{e}_n) -
            C({\rho_\out}_1^\ast, \overline{e}_n)\\
            C({\rho_\out}_i^\ast, \overline{e}_1) -
            C({\rho_\out}_2^\ast, \overline{e}_1) &
            \cdots &
            C({\rho_\out}_i^\ast, \overline{e}_n) -
            C({\rho_\out}_2^\ast, \overline{e}_n)\\
            \vdots & \vdots & \vdots \\
            C({\rho_\out}_i^\ast, \overline{e}_1) -
            C({\rho_\out}_N^\ast, \overline{e}_1) &
            \cdots &
            C({\rho_\out}_i^\ast, \overline{e}_n) -
            C({\rho_\out}_N^\ast, \overline{e}_n)\\
            C({\rho_\out}_i^\ast, \overline{e}_1) -
            C^\ast(\overline{e}_1) &
            \cdots &
            C({\rho_\out}_i^\ast, \overline{e}_n) -
            C^\ast(\overline{e}_n)
    \end{array}\right] \label{eq:A_defn} \\
        \overline{b}_i&={[0\ 0\ \ldots\ 0\
        \epsilon]}^\top\in \mathbb{R}^{N+1}\label{eq:b_defn}
\end{align}}\normalsize%
with $H_i\in \mathbb{R}^{(N+1)\times n}$. By \eqref{prob:opt} and Assumption~\ref{assum:C_lin}, we have the following upper bound on $C^\ast(\overline{p})$ for any runtime information vector $ \overline{p}={[q_1\ q_2\ \ldots\ q_n]}^\top\in
\mathcal{P}$,
\begin{align}
    C^\ast( \overline{p})\leq
    C({\rho_\out}_i^\ast,
    \overline{p})=\sum_{j=1}^nC({\rho_\out}_i^\ast,
    \overline{e}_j)q_j\label{eq:C_ub},\quad \forall
    i\in\{1,\ldots,N\}.
\end{align}

Theorem~\ref{thm:bounds} below ensures that using ${\rho_\out^\ast}_i$ for vectors $\overline p \in \mathcal{S}_i$ results in $\epsilon$-optimal performance, provided the polytopes $\mathcal{S}_i$ are non-empty. In other words, the difference between $C^\ast(\overline p)$, the optimal performance for a runtime information $\overline{p}$,  and $C(\rho_{\out_i}^\ast, \overline p)$, the attained performance due to choice of strategy $\rho_{\out_i}^\ast$,  can not be larger than $\epsilon$, whenever $\overline p \in \mathcal{S}_i$. To complete the discussion, we provide a sufficient condition for non-empty polytopes $\mathcal{S}_i$ in Proposition~\ref{prop:nonempty}.
 
\begin{thm}\label{thm:bounds}
    Let the polytopes $\mathcal{S}_i$ be non-empty. Given runtime information $ \overline{p} \in \mathbb{R}^n$,
    if $ \overline{p}\in \mathcal{S}_i$ for some
    $i\in\{1,\ldots, N\}$, then
    \begin{align}
        C({\rho_\out}_i^\ast,
    \overline{p}) - \epsilon \leq
        C^\ast(\overline{p}) \leq
        C({\rho_\out}_i^\ast,
    \overline{p}). \nonumber
    \end{align}
\end{thm}
\begin{proof}
   The upper bound on $C^\ast( \overline{p})$ follows from
   \eqref{eq:C_ub}. Consider the collection of polytopes $
    \mathcal{T}_i$ constructed using the first $N$
    hyperplanes in \eqref{eq:A_defn} and \eqref{eq:b_defn}.
    For every $ \overline{p}\in \mathcal{T}_i$,
    \begin{align}
        C({\rho_\out}_i^\ast, \overline{p})&\leq
        C({\rho_\out}_k^\ast, \overline{p}),\ \forall
        k\in\{1,\ldots,N\}\setminus\{i\}\label{eq:dom}.
    \end{align}
    In other words, among the $N$  strategies
    ${\rho_\out}_{(\cdot)}^\ast$,
    ${\rho_\out}_{i}^\ast$ provides the tightest upper
    bound on $C^\ast( \overline{p})$ due to \eqref{eq:C_ub}.

   We next prove the lower bound. The last
   hyperplane in \eqref{eq:A_defn} and \eqref{eq:b_defn}
   guarantees that $ C(
   {\rho_\out}_i^\ast, \overline{p}) -  \overline{\ell}^\top \overline{p} \leq \epsilon$ for
   every $ \overline{p}\in \mathcal{S}_i$, where $\overline{\ell}_j = C^\ast( \overline{e}_j)$. On adding and
   subtracting $C^\ast( \overline{p})$, we have
   $C({\rho_\out}_i^\ast, \overline{p}) - C^\ast(
   \overline{p}) + C^\ast(\overline{p}) -
   \overline{\ell}^\top \overline{p} \leq \epsilon$. Since
   $C( \cdot, \overline{e}_j)\geq \ell_j$ by definition of $
   \overline{\ell}$, we have $C^\ast(\overline{p}) -
   \overline{\ell}^\top \overline{p} \geq 0$ for every $
   \overline{p}\in \mathcal{S}_j$. Therefore, $C({\rho_\out}_i^\ast, \overline{p}) - C^\ast(
   \overline{p}) \leq \epsilon$.
\end{proof}
\begin{prop}\label{prop:nonempty}
    Let Assumption~\ref{assum:C_lin} hold. For every $i=1,\ldots,N$ the polytope $\mathcal{S}_i$ is non-empty, provided that $\epsilon \geq \max_i\left\{ C( {\rho_\out}_i^\ast, \overline{p}_i) -
    \overline{\ell}^\top \overline{p}_i\right\}$. Here,
    $\overline{\ell}={[C^\ast( \overline{e}_1)\ C^\ast(\overline{e}_2)\ \ldots\ C^\ast( \overline{e}_n)]}^\top\in \mathbb{R}^n$.
\end{prop}
\begin{proof}
    The polytopes $ \mathcal{T}_i$ (defined in the proof of Theorem~\ref{thm:bounds}) are non-empty, since they
    contain $ \overline{p}_i$ by the optimality of $
    {\rho_\out}_i^\ast$ in \eqref{prob:opt}.
    The last hyperplane in \eqref{eq:A_defn} and
    \eqref{eq:b_defn} is also satisfied by $
    \overline{p}_i$, thanks to the use of $\epsilon$ in $\overline b_i$.
    Thus, its intersection with $ \mathcal{T}_i$, which
    yields the polytope $ \mathcal{S}_i$, is non-empty.
\end{proof}

%% file: switching.tex
In order to guarantee that the plays resulting from switching between the synthesized representative strategies satisfy the specification $\spec$, the switching function needs to keep track the satisfaction of the agent's liveness guarantees $\LTLglobally\LTLfinally F_i$ in $\spec$. Since the cost function $C$ captures the cost of achieving the liveness guarantees, when the specification is satisfied due to a violation of the environment assumptions, no switching would be necessary, as the cost would be $0$.

To ensure that the switching between strategies does not prevent the agent from infinitely often visiting each of the sets $F_i$, the switching function will keep track of these visits, and only allow switching to a different strategy once all the sets $F_i$ have been visited under the current strategy. Furthermore, the switch can only occur from a state from which the next strategy can guarantee the satisfaction of $\spec$. Below we make this intuition precise by providing the construction of the switching function as a finite-state system.

Let $\{{\rho_\out^\ast}_i\in \mathcal{M}_\out\}_{i=1}^{N}$ be a set of strategies  for the agent such that ${\rho_\out}_i^\ast \models \spec$ for each $i \in \{1,\ldots,N\}$, and let $\{\mathcal S_i\}_{i=1}^N$ be the polytopes computed as in Section~\ref{sec:generalization}.

For each ${\rho_\out^\ast}_i$, let $W_i = \{g\in G \mid \plays(\game,{\rho_\out^\ast}_i,g)\subseteq \spec\}$ denote the set of states from which the specification can be enforced by following the strategy ${\rho_\out^\ast}_i$.

We define a finite state transition system with states $Q$, initial state $q_0$, transition function $\theta$ and alphabet $(G \times \ialphabet \times \oalphabet \times G) \times \{\mathcal S\}_{i=1}^N$. The set of states is $Q = \{ V \mid V \subseteq \{1,\ldots, n\}\} \times \{1,\ldots,N\}$, where $n$ is the number of liveness guarantees in $\spec$. States in $Q$ track the guarantees that have been satisfied and contain the index of the currently chosen strategy. The initial state is $q_0 = (\{1,\ldots,n\},1)$. The transition function $\theta : Q \times ((G \times \ialphabet \times \oalphabet \times G) \times \{\mathcal S\}_{i=1}^N) \to Q$ is defined such that 
$\theta((V,i),((g,\symb_{\inp}, \symb_{\out}, g'),\mathcal S)) = (V',i')$, where
\begin{itemize}
\item if $V = \{1,\ldots,n\}, \mathcal S = \mathcal S_{j}, g' \in W_{j}$, then, $V' = \emptyset, i'=j$,
\item $V' = V \cup \{j \in \{1,\ldots,n\} \mid g \in F_j\}$ and $i'=i$ otherwise.
\end{itemize}
That is, once all the sets $F_j$ have been visited under the current strategy, we can switch to the $i'$-th strategy and reset the tracking set to $\emptyset$. Otherwise we record the indices of the visited sets and keep the strategy index the same. We can extend $\theta$ to words in the usual way. 

We define the switching function such that 
$\tau (\varepsilon,\overline p_0) = \min(\{i \mid p_0 \in \mathcal S_i\}\cup \{N\})$, where $\varepsilon$ is the empty word, and for every 
$\overline \pi =(g_0,\symb_{\inp,0},\symb_{\out,0}, g_1) \ldots
(g_k,\symb_{\inp,k}, \symb_{\out,k}, g_{k+1})$ and every $\overline \gamma = \overline p_0\ldots\overline p_{k+1}$ we let
$\tau (\overline \pi, \overline\gamma) = i$ where
$\theta(((g_0,\symb_{\inp,0},\symb_{\out,0}, g_1),\mathcal S_{i_1}) \ldots
((g_k,\symb_{\inp,k}, \symb_{\out,k}, g_{k+1}),\mathcal S_{i_{k+1}})) = (V,i)$ for some $V$, where for all $j\geq 1$ we have $i_j  = \min(\{i \mid \overline p_j \in \mathcal S_i\text{ and }g_j \in W_{i_j}\} \cup \{N\})$.

As the switching function $\tau$ only allows switching to a different strategy once all $F_j$ have been visited under the current strategy, this means that if we switch strategy infinitely often then the liveness guarantee is satisfied. If, on the other hand, we stabilize at some strategy, $\spec$ is again guaranteed by the fact that this strategy satisfies the specification.

%% file: Switching_Discussion.tex
The key advantage of our approach is that we avoid the re-synthesis of strategies for each new value of the runtime information, when the parameter set is covered by the collection of polytopes ${\{\mathcal{S}_i\}}_{i=1}^N$,  $\mathcal{P}\subseteq\cup_{i=1}^N \mathcal{S}_i$.  We also avoid discretization of the set $\mathcal{P}$. This enables on-board deployment of our approach with guaranteed $\epsilon$-optimal performance. In contrast, traditional approaches rely either on re-synthesis or on discretization of the parameter space which requires either prohibitively high computational or memory costs~\cite{jangcontinuous,smith2011optimal}. When $\mathcal{P}\not\subseteq\cup_{i=1}^N \mathcal{S}_i$, none of the synthesized strategies guarantees $\epsilon$-optimal performance for the parameter values in $\mathcal{P}\setminus\cup_{i=1}^N \mathcal{S}_i$. In such cases, we can iteratively expand the candidate instantiations offline till the entire parameter space $\mathcal{P}$ is covered. Specifically, we add to the candidate instantiations randomly chosen parameter values in $\mathcal{P}\setminus\cup_{i=1}^N \mathcal{S}_i$. In future, we intend to investigate sufficient conditions under which such an expansion approach terminates in finite number of steps.

%% file: figs/presentationfigure.tex
%
%
\definecolor{mycolor1}{rgb}{1.00000,0.00000,1.00000}%
\definecolor{mycolor2}{rgb}{0.00000,1.00000,1.00000}%
\begin{tikzpicture}[scale=0.95]

\begin{axis}[%
width=0.8\linewidth,
at={(5.308in,1.073in)},
scale only axis,
plot box ratio=1 1 1,
xmin=-0.05,
xmax=1.2,
xlabel = {$p_1$},
tick align=outside,
ymin=-0.05,
ymax=1.2,
ylabel = {$p_2$},
zmin=-0.05,
zmax=1.2,
zlabel = {$p_3$},
view={105}{35},
axis line style={draw=none},
ticks=none,
axis x line*=bottom,
axis y line*=left,
axis z line*=left,
xmajorgrids,
ymajorgrids,
zmajorgrids,
legend cell align={left},
legend style={at={(0.6,0.95)}, anchor=north west, draw=white!0.0!black},
]

\addplot3[area legend, line width=1.0pt, draw=black, fill=red, fill opacity=0, forget plot]
table[row sep=crcr] {%
x	y	z\\
0.999999999999999	3.88578058618805e-16	4.44089209850063e-16\\
1.66533453693773e-16	1	3.33066907387547e-16\\
4.44089209850063e-16	3.88578058618805e-16	0.999999999999999\\
}--cycle;
\addplot3[area legend, dotted, line width=1.0pt, draw=black, fill=green, fill opacity=0.05, forget plot]
table[row sep=crcr] {%
x	y	z\\
0.5	4.16333634234434e-16	0.5\\
0.999999999999999	4.16333634234434e-16	3.60822483003176e-16\\
0.5	0.5	4.71844785465692e-16\\
0.333333333333333	0.333333333333333	0.333333333333333\\
}--cycle;

\addplot3[area legend, dashed, line width=1.0pt, draw=black, fill=green, fill opacity=0.3, forget plot]
table[row sep=crcr] {%
x	y	z\\
0.999999999999999	1.52655665885959e-16	4.85722573273506e-16\\
0.59875	1.52655665885959e-16	0.40125\\
0.799375	0.200625	4.85722573273506e-16\\
}--cycle;

\addplot3[area legend, dotted, line width=1.0pt, draw=black, fill=mycolor1, fill opacity=0.05, forget plot]
table[row sep=crcr] {%
x	y	z\\
0.5	0.5	-3.88578058618805e-16\\
4.71844785465692e-16	1	-3.88578058618805e-16\\
3.88578058618805e-16	0.333333333333333	0.666666666666666\\
0.333333333333333	0.333333333333333	0.333333333333333\\
}--cycle;

\addplot3[area legend, dashed, line width=1.0pt, draw=black, fill=mycolor1, fill opacity=0.3, forget plot]
table[row sep=crcr] {%
x	y	z\\
-3.46944695195361e-16	1	1.38777878078145e-17\\
-2.4980018054066e-16	0.59875	0.401250000000001\\
0.200625	0.799375	1.38777878078145e-17\\
}--cycle;

\addplot3[area legend, dotted, line width=1.0pt, draw=black, fill=mycolor2, fill opacity=0.05, forget plot]
table[row sep=crcr] {%
x	y	z\\
0.5	-8.32667268468867e-17	0.5\\
2.22044604925031e-16	-5.55111512312578e-17	1\\
2.22044604925031e-16	0.333333333333333	0.666666666666666\\
0.333333333333333	0.333333333333333	0.333333333333333\\
}--cycle;

\addplot3[area legend, dashed, line width=1.0pt, draw=black, fill=mycolor2, fill opacity=0.3, forget plot]
table[row sep=crcr] {%
x	y	z\\
2.4980018054066e-16	-3.33066907387547e-16	1\\
2.22044604925031e-16	0.200625	0.799375\\
0.401250000000001	-3.88578058618805e-16	0.59875\\
}--cycle;
\addplot3[only marks, mark=*, mark options={}, mark size=3.5pt, color=black, fill=green] table[row sep=crcr]{%
x	y	z\\
0.7	0.1	0.2\\
};\addlegendentry{$\overline{p_1} = \lbrack 0.7, 0.1, 0.2 \rbrack$}

\addplot3[only marks, mark=*, mark options={}, mark size=3.5pt, color=black, fill=mycolor1] table[row sep=crcr]{%
x	y	z\\
0.1	0.7	0.2\\
};\addlegendentry{$\overline{p_2} = \lbrack 0.1, 0.7, 0.2 \rbrack$}

\addplot3[only marks, mark=*, mark options={}, mark size=3.5pt, color=black, fill=mycolor2] table[row sep=crcr]{%
x	y	z\\
0.2	0.1	0.7\\
};\addlegendentry{$\overline{p_3} = \lbrack 0.2, 0.1, 0.7 \rbrack$}

\addplot3[->, color=black, line width = 2.0pt,point meta={sqrt((\thisrow{u})^2+(\thisrow{v})^2+(\thisrow{w})^2)}, point meta min=0, quiver={u=\thisrow{u}, v=\thisrow{v}, w=\thisrow{w}, every arrow/.append}]
 table[row sep=crcr] {%
x	y	z	u	v	w\\
0	0	0	1.125	0	0\\
};
\addplot3[->, color=black, line width = 2.0pt, point meta={sqrt((\thisrow{u})^2+(\thisrow{v})^2+(\thisrow{w})^2)}, point meta min=0, quiver={u=\thisrow{u}, v=\thisrow{v}, w=\thisrow{w}, every arrow/.append }]
 table[row sep=crcr] {%
x	y	z	u	v	w\\
0	0	0	0	1.125	0\\
};
\addplot3[->, color=black,line width = 2.0pt, point meta={sqrt((\thisrow{u})^2+(\thisrow{v})^2+(\thisrow{w})^2)}, point meta min=0, quiver={u=\thisrow{u}, v=\thisrow{v}, w=\thisrow{w}, every arrow/.append}]
 table[row sep=crcr] {%
x	y	z	u	v	w\\
0	0	0	0	0	1.125\\
};
\end{axis}
\end{tikzpicture}%

%% file: root.bbl
\begin{thebibliography}{10}
\providecommand{\url}[1]{#1}
\csname url@samestyle\endcsname
\providecommand{\newblock}{\relax}
\providecommand{\bibinfo}[2]{#2}
\providecommand{\BIBentrySTDinterwordspacing}{\spaceskip=0pt\relax}
\providecommand{\BIBentryALTinterwordstretchfactor}{4}
\providecommand{\BIBentryALTinterwordspacing}{\spaceskip=\fontdimen2\font plus
\BIBentryALTinterwordstretchfactor\fontdimen3\font minus
  \fontdimen4\font\relax}
\providecommand{\BIBforeignlanguage}[2]{{%
\expandafter\ifx\csname l@#1\endcsname\relax
\typeout{** WARNING: IEEEtran.bst: No hyphenation pattern has been}%
\typeout{** loaded for the language `#1'. Using the pattern for}%
\typeout{** the default language instead.}%
\else
\language=\csname l@#1\endcsname
\fi
#2}}
\providecommand{\BIBdecl}{\relax}
\BIBdecl

\bibitem{belta2007symbolic}
C.~Belta, A.~Bicchi, M.~Egerstedt, E.~Frazzoli, E.~Klavins, and G.~J. Pappas,
  ``Symbolic planning and control of robot motion [grand challenges of
  robotics],'' \emph{IEEE Robotics \& Automation Magazine}, vol.~14, no.~1, pp.
  61--70, 2007.

\bibitem{finucane2010ltlmop}
C.~Finucane, G.~Jing, and H.~Kress-Gazit, ``Ltlmop: Experimenting with
  language, temporal logic and robot control,'' in \emph{2010 IEEE/RSJ
  International Conference on Intelligent Robots and Systems}, 2010, pp.
  1988--1993.

\bibitem{Ehlerscost}
G.~Jing, R.~Ehlers, and H.~Kress{-}Gazit, ``Shortcut through an evil door:
  Optimality of correct-by-construction controllers in adversarial
  environments,'' in \emph{2013 {IEEE/RSJ} International Conference on
  Intelligent Robots and Systems, Tokyo, Japan, November 3-7, 2013}, 2013, pp.
  4796--4802.

\bibitem{jangcontinuous}
G.~Jing and H.~Kress{-}Gazit, ``Improving the continuous execution of reactive
  ltl-based controllers,'' in \emph{2013 {IEEE} International Conference on
  Robotics and Automation, Karlsruhe, Germany, May 6-10, 2013}, 2013, pp.
  5439--5445.

\bibitem{goyal2018urban}
R.~Goyal, ``Urban air mobility (uam) market study,'' 2018.

\bibitem{gipson2017nasa}
L.~Gipson, ``Nasa embraces urban air mobility, calls for market study,''
  \emph{NASA. November}, vol.~7, 2017.

\bibitem{thipphavong2018urban}
D.~P. Thipphavong, R.~Apaza, B.~Barmore, V.~Battiste, B.~Burian, Q.~Dao,
  M.~Feary, S.~Go, K.~H. Goodrich, J.~Homola, H.~R. Idris, P.~H. Kopardekar,
  J.~B. Lachter, N.~A. Neogi, H.~K. Ng, R.~M. Oseguera-Lohr, M.~D. Patterson,
  and S.~A. Verma, ``Urban air mobility airspace integration concepts and
  considerations,'' in \emph{2018 Aviation Technology, Integration, and
  Operations Conference}, 2018, p. 3676.

\bibitem{bhnfm}
S.~Bharadwaj, S.~Carr, N.~Neogi, H.~Poonawala, A.~B. Chueca, and U.~Topcu,
  ``Traffic management for urban air mobility,'' in \emph{{NASA} Formal Methods
  - 11th International Symposium, {NFM} 2019, Houston, TX, USA, May 7-9, 2019,
  Proceedings}, 2019, pp. 71--87.

\bibitem{AlshiekhShield}
M.~Alshiekh, R.~Bloem, R.~Ehlers, B.~Könighofer, S.~Niekum, and U.~Topcu,
  ``Safe reinforcement learning via shielding,'' in \emph{AAAI Conference on
  Artificial Intelligence}, 2018.

\bibitem{Konighofer2017}
B.~K{\"o}nighofer, M.~Alshiekh, R.~Bloem, L.~Humphrey, R.~K{\"o}nighofer,
  U.~Topcu, and C.~Wang, ``Shield synthesis,'' \emph{Formal Methods in System
  Design}, vol.~51, no.~2, pp. 332--361, Nov 2017.

\bibitem{8815233}
S.~{Bharadwaj}, R.~{Bloem}, R.~{Dimitrova}, B.~{Konighofer}, and U.~{Topcu},
  ``Synthesis of minimum-cost shields for multi-agent systems,'' in
  \emph{American Control Conference (ACC)}, July 2019, pp. 1048--1055.

\bibitem{piterman2006}
N.~Piterman, A.~Pnueli, and Y.~Sa'ar, ``Synthesis of reactive(1) designs,'' in
  \emph{Verification, Model Checking, and Abstract Interpretation, 7th
  International Conference, {VMCAI} 2006, Charleston, SC, USA, January 8-10,
  2006, Proceedings}, 2006, pp. 364--380.

\bibitem{bloem2012}
R.~Bloem, B.~Jobstmann, N.~Piterman, A.~Pnueli, and Y.~Sa'ar, ``Synthesis of
  reactive(1) designs,'' \emph{J. Comput. Syst. Sci.}, vol.~78, no.~3, pp.
  911--938, 2012.

\bibitem{Ehlerslugs}
R.~Ehlers and V.~Raman, ``Slugs: Extensible {GR(1)} synthesis,'' in
  \emph{Computer Aided Verification - 28th International Conference, {CAV}
  2016, Toronto, ON, Canada, July 17-23, 2016, Proceedings, Part {II}}, 2016,
  pp. 333--339.

\bibitem{bh18}
S.~Bharadwaj, R.~Dimitrova, and U.~Topcu, ``Synthesis of surveillance
  strategies via belief abstraction,'' \emph{CoRR}, vol. abs/1709.05363, 2017,
  \url{http://arxiv.org/abs/1709.05363}.

\bibitem{Alonso18}
J.~Alonso{-}Mora, J.~A. DeCastro, V.~Raman, D.~Rus, and H.~Kress{-}Gazit,
  ``Reactive mission and motion planning with deadlock resolution avoiding
  dynamic obstacles,'' \emph{Auton. Robots}, vol.~42, no.~4, pp. 801--824,
  2018.

\bibitem{Moarref18}
S.~Moarref and H.~Kress{-}Gazit, ``Reactive synthesis for robotic swarms,'' in
  \emph{Formal Modeling and Analysis of Timed Systems - 16th International
  Conference, {FORMATS} 2018, Beijing, China, September 4-6, 2018,
  Proceedings}, 2018, pp. 71--87.

\bibitem{MCBook}
C.~Baier and J.-P. Katoen, \emph{Principles of model checking}.\hskip 1em plus
  0.5em minus 0.4em\relax MIT press, 2008.

\bibitem{smith2011optimal}
S.~L. Smith, J.~T{\'u}mov{\'a}, C.~Belta, and D.~Rus, ``Optimal path planning
  for surveillance with temporal-logic constraints,'' \emph{The International
  Journal of Robotics Research}, vol.~30, no.~14, pp. 1695--1708, 2011.

\end{thebibliography}
